\def\year{2022}\relax
\documentclass[letterpaper]{article} 
\usepackage{aaai22}  
\usepackage{times}  
\usepackage{helvet}  
\usepackage{courier}  
\usepackage[hyphens]{url}  
\usepackage{graphicx} 
\usepackage{natbib}  
\usepackage{caption} 
\DeclareCaptionStyle{ruled}{labelfont=normalfont,labelsep=colon,strut=off} 
\usepackage[linesnumbered,lined,commentsnumbered,ruled,vlined]{algorithm2e}
\usepackage{xcolor}
\usepackage{amsmath}
\usepackage{amsfonts}
\usepackage{amsthm}
\newtheorem{theorem}{Theorem}
\DeclareMathOperator*{\minimize}{minimize}
\usepackage{footnote}

\title{Strategy Discovery and Mixture in Lifelong Learning from Heterogeneous Demonstration}

\author{
    Sravan Jayanthi\equalcontrib{\renewcommand\footnotemark{}\thanks{Accepted at the AAAI-22 Workshop on Interactive Machine Learning (IML@AAAI'22)}},
    Letian Chen\equalcontrib,
    Matthew Gombolay
}
\affiliations{

    School of Interactive Computing\\
    Georgia Institute of Technology\\
    Atlanta, GA, 30332\\
    sjayanthi@gatech.edu, letian.chen@gatech.edu, matthew.gombolay@cc.gatech.edu
}

\begin{document}
\maketitle

\begin{abstract}
Learning from Demonstration (LfD) approaches empower end-users to teach robots novel tasks via demonstrations of the desired behaviors, democratizing access to robotics. A key challenge in LfD research is that users tend to provide heterogeneous demonstrations for the same task due to various strategies and preferences. Therefore, it is essential to develop LfD algorithms that ensure \textit{flexibility} (the robot adapts to personalized strategies), \textit{efficiency} (the robot achieves sample-efficient adaptation), and \textit{scalability} (robot reuses a concise set of strategies to represent a large amount of behaviors). In this paper, we propose a novel algorithm, Dynamic Multi-Strategy Reward Distillation (DMSRD), which distills common knowledge between heterogeneous demonstrations, leverages learned strategies to construct mixture policies, and continues to improve by learning from all available data. Our personalized, federated, and lifelong LfD architecture surpasses benchmarks in two continuous control problems with an average 77\% improvement in policy returns and 42\% improvement in log likelihood, alongside stronger task reward correlation and more precise strategy rewards. 
\end{abstract}

\section{Introduction}
\quad Robotics has made significant progress owing to recent advancements of Deep Reinforcement Learning (DRL) techniques on training complex control tasks~\cite{DBLP:journals/corr/abs-1812-05905,seraj2021heterogeneous,konan2022iterated}. However, RL's success heavily relies on sophisticated reward functions designed for each task, which requires the input and expertise of RL researchers to be effective~\cite{10.1007/11840817_87, seraj2021hierarchical}, limiting the scalability and ubiquity of robots. Even RL researchers may fail to engineer a proper reward function that generates behaviors matching their latent expectations~\cite{Amodei2016ConcretePI}, and a more time-consuming trial-and-error process has to be employed to discover a suitable reward function. 


Instead, Learning from Demonstration (LfD) approaches seek to democratize access to robotics by having users demonstrate the desired task to the robot~\cite{NIPS1996_68d13cf2,gombolay2018robotic,palan2019learning}, making robot learning more readily accessible to non-roboticist end-users. As LfD research strives to empower end-users with the ability to program novel behaviors for robots, we must consider that end-users adopt varying preferences and strategies in how they complete a task. For example, when teaching a table tennis robot, one user might show a topspin strike while another a lob shot, as illustrated in Figure~\ref{fig:DMSRD_illustration}. A third user could demonstrate a topspin-lob strike, seeking to add both spin and height to the shot. This heterogeneity in demonstrations could fail the machine learning algorithm~\cite{morales2004learning} and the robot could approach the task in such a way that is undesirable for the end-user.

Research has shown robots that personalize their behaviors according to users achieve more engagement, long-term acceptance, and performance~\cite{lesort2020continual, paleja2021utility}. In order to realize personalization, the robot could query her user for demonstrations and train a new model from scratch with the homogeneous data. Nonetheless, it is inefficient to require the user to provide a large number of demonstrations, and limited data would inhibit precise learning of the user intentions. An ideal LfD algorithm should not only be adaptive but also leverage knowledge through the lifelong learning of approaches for a task (e.g., from previous users) to efficiently capture a novel strategy.


\begin{figure*}[t!]
\centering
\includegraphics[width=0.85\textwidth]{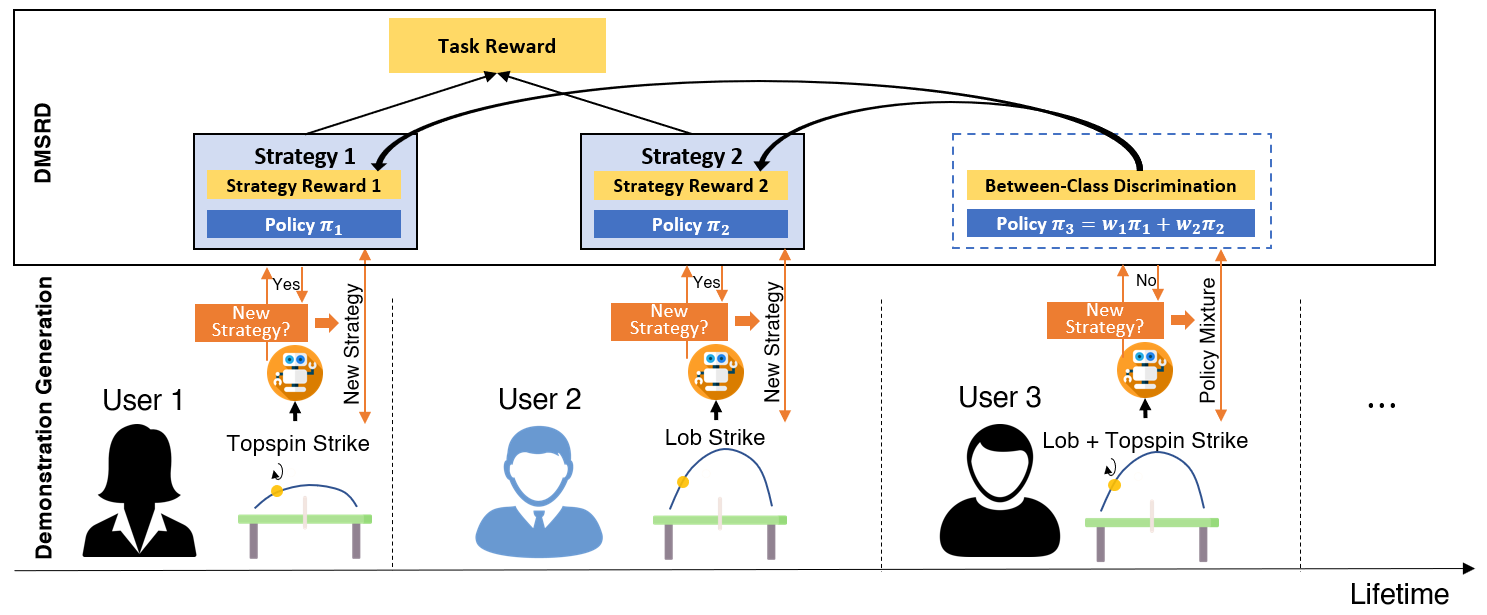} \\
\caption{This figure illustrates the online learning process with our proposed method DMSRD.}
\label{fig:DMSRD_illustration}
\end{figure*}
In this paper, we propose a novel LfD framework, Dynamic Multi-Strategy Reward Distillation (DMSRD), to accomplish \textit{flexible}, \textit{efficient}, and \textit{scalable} lifelong learning from heterogeneous demonstration. In a lifelong learning process where demonstrations arrive in sequence, DMSRD distills common knowledge across demonstrations, autonomously creates new strategies or mixtures of learned strategies, and organizes knowledge from all demonstrations. The federated and personalized architecture enables DMSRD to learn how to succeed at the task while personalizing for each user. DMSRD realizes the scalability needed for lifelong learning by reusing already-synthesized policies as \textit{policy mixtures} to model new strategies when applicable. 

Our contribution in this paper is three-fold: 

\begin{enumerate}
    \item We propose DMSRD, an IRL framework that can dynamically learn a concise set of strategies that model a range of heterogeneous behaviors demonstrated by users.
    \item We show the success of DMSRD on two continuous control tasks and find that DMSRD is able to learn robust task rewards (correlations of $0.953$ and $0.614$ with the ground-truth rewards, respectively) and strategy rewards that encode specific preferences. DMSRD learns policies that achieve higher returns on the task (77\% and 76\% improvement) and better capture the strategic preference than best benchmarks (43\% and 40\% improvement).
    \item We demonstrate that DMSRD scales well with more users in an experiment with 100 demonstrations. DMSRD identifies seven strategies from the 100 demonstrations and utilizes \textit{policy mixtures} to achieve a precise representation of all demonstrations. 
\end{enumerate}

\section{Related Work}
\quad Several approaches for LfD have been proposed, including Imitation Learning (IL)~\cite{Paleja2019InterpretableAP,wangheterogeneous,silva2021lancon} and Inverse Reinforcement Learning (IRL)~\cite{chencorl2020, chen2021towards}. Unlike IL's direct learning of a mapping from states to demonstrated actions, IRL 
infers the latent reward functions the demonstrators optimize for. Adversarial Inverse Reinforcement Learning (AIRL) disentangles reward functions from the environment dynamics, making the recovered reward a robust representation of the demonstrator's intent~\cite{fu2017learning}. Our work leverages AIRL as the backbone IRL method when the available knowledge is not sufficient to explain a new demonstration.


Although traditional IRL approaches usually overlook heterogeneity within demonstrations~\cite{abbeel2004apprenticeship,ramachandran2007bayesian,ziebart2008maximum}, there has been recent work on modeling heterogeneous demonstrations. One intuitive way is to classify demonstrations into homogeneous clusters before applying IRL~\cite{nikolaidis2015efficient}. In these approaches, each reward function only learns from a portion of the demonstrations, making the learning process prone to reward ambiguity~\cite{Chen2020JointGA}. To mitigate the issue, Expectation-Maximization (EM) algorithm has been proposed to make soft assignments of demonstrations to strategies. EM methods treat clustering as the E-step and the IRL problem as the M-step~\cite{vroman,pmlr-v108-ramponi20a}. When the number of strategies is unknown, a Dirichlet Process prior~\cite{NIPS2012_140f6969,10.1007/978-3-030-86486-6_13} or non-parametric methods~\cite{nonparametricbehavior} could be used. While this prior work realizes the \textit{flexibility} to adapt to heterogeneous demonstrations, none consider a lifelong learning from demonstration problem, where \textit{efficiency} and \textit{scalability} are essential. Moreover, previous work mainly operates on a featured state space~\cite{10.1007/978-3-030-86486-6_13,NIPS2012_140f6969} to guarantee tractable EM while our method avails itself to high dimensional state spaces by neural network. Despite the abundance of previous methods, none consider relationships between policies learned. Our method, DMSRD, exploits these relationships to not only model heterogeneous demonstrations, but does so \textit{efficiently} by leveraging already-learned policies and can scale to model large number of demonstrations (\textit{scalability}). 

A previous framework that achieves personalization while considering commonalities between strategies (i.e., the task) is Multi-Strategy Reward Distillation (MSRD). MSRD decomposes person-specific reward functions into a shared task reward and personalized strategy rewards, achieving accurate recovery of the task objective and demonstrators' preferences. Despite MSRD's success, it has a few drawbacks. First, MSRD is only compatible with offline training, thus unable to achieve continual improvements if new demonstrations become available. Second, MSRD assumes access to the strategy labels for each demonstration, which is cost-ineffective, error-prone, and decreases \textit{scalability}. Third, MSRD only consider relationships between demonstrators' reward functions, leaving much to be desired on the policy adaptation efficiency to a new user. DMSRD leverages these relationships to identify a set of base policies that can explain all strategies. Additionally, DMSRD's reward functions learn to identify partial-strategy demonstrations and recover more accurate strategy preferences.

\section{Preliminaries}
In this section, we introduce preliminaries on Markov Decision Process, Inverse Reinforcement Learning, and Multi-Strategy Reward Distillation. 

\paragraph{Markov Decision Process}
A Markov Decision Process (MDP) $M$ is described as a 6-tuple $\langle \mathbb{S},\mathbb{A},R,T,\gamma,\rho_0\rangle$. $\mathbb{S}$ and $\mathbb{A}$ represent the state space and the action space, respectively. $R: \mathbb{S}\rightarrow \mathbb{R}$ denotes the reward function, meaning the environment provides $R(s)$ reward for state $s$. $T: \mathbb{S}\times \mathbb{A}\times \mathbb{S}\rightarrow \mathbb{R}$ represents the transition function, and $T(s^\prime\vert s,a)$ is the probability of transitioning into state $s^\prime$ after taking action $a$ in state $s$. $\gamma\in(0,1)$ is the temporal discount factor. $\rho_0: \mathbb{S}\rightarrow \mathbb{R}$ denotes the initial state probability. A policy, $\pi: \mathbb{S}\times\mathbb{A}\rightarrow \mathbb{R}$, represents the probability to choose an action given the state. An infinite-horizon MDP alternates between policy executions and environment transitions: $s_1\sim \rho_0(\cdot),\ a_t\sim\pi(\cdot\vert s_t),\ s_{t+1}\sim T(\cdot\vert s_t,a_t),\ t=1,2,\cdots$. The standard objective for RL is to find the optimal policy that maximizes the discounted return, $\pi^*=\arg\max_\pi\mathbb{E}_{\tau\sim\pi}\left[\sum_{t=1}^\infty{\gamma^{t-1}R(s_t)}\right]$, and $\tau\sim\pi$ means the trajectory $\tau$ is induced by the policy $\pi$ and implicitly contains the state-action sequence: $\{s_1,a_1,s_2,a_2,\cdots\}$. 

\paragraph{Inverse Reinforcement Learning}
Inverse Reinforcement Learning (IRL) considers an MDP sans reward function ($M\backslash R$) and aims to infer the reward function $R$ based on a set of demonstrations $\mathcal{U}=\{\tau_1,\tau_2,\cdots,\tau_N\}$ where $N$ is the number of demonstrations. Our method is based on Adversarial Inverse Reinforcement Learning (AIRL)~\cite{fu2017learning}, which solves the IRL problem with a generative adversarial setup. The discriminator, $D: \mathbb{S}\times\mathbb{A}\times\mathbb{S}\rightarrow\mathbb{R}$, predicts the probability that the transition $(s_t,a_t,s_{t+1})$ belongs to the demonstration rather than rollouts generated by the generator policy, $\pi_G(a|s)$. The discriminator, $D_\theta$, is trained by the cross-entropy loss: $L_D=-\mathbb{E}_{\tau\sim\mathcal{U}}\left[\log D_\theta(s_t,a_t,s_{t+1})\right]-\mathbb{E}_{\tau\sim\pi_G}\left[{\log(1-D_\theta(s_t,a_t,s_{t+1}))}\right]$. AIRL further proposes a special parameterization of the discriminator, $D(s,a,s^\prime)=\frac{\exp\left(g(s)+\gamma h(s^\prime)-h(s)\right)}{\exp\left(g(s)+\gamma h(s^\prime)-h(s)\right)+\pi(a\vert s)}$. Under certain assumptions, AIRL can recover $g$ as the ground-truth reward, $R$, and $h$ as the ground-truth value, $V$. The generator policy $\pi_G$ is trained via policy gradient to maximize the pseudo-reward given by $g(s)$, and at the global optimal point, $\pi_G$ will recover the demonstration policy. 


\paragraph{Multi-Strategy Reward Distillation} MSRD makes the strong assumption that the strategy label, $c_{\tau_i}\in\{1,\cdots,M\}$, for each demonstration $\tau_i$ is available, where $M$ is the number of strategies. Utilizing the strategy label information, MSRD categorizes demonstrations into strategies, and decomposes the per-strategy reward $R_i$ (for strategy $i$) as a linear combination of task reward, $R_{\text{Task}}$, and strategy-only reward, $R_{\text{S-}i}$: $R_{i}=R_{\text{Task}}+\alpha R_{\text{S-}i}$, where $\alpha$ is a relative weight between task reward and strategy-only reward. MSRD parameterizes the task reward by $\theta_{\text{Task}}$ and strategy-only reward by $\theta_{\text{S}-i}$, takes AIRL as its backbone IRL method, and adds a regularization to encourage knowledge distillation into the task reward $\theta_\text{Task}$ and only keep personalized information in $\theta_{\text{S}-i}$, as shown by the MSRD loss in Equation~\ref{eq:MSRD}. 
\begin{equation}
\label{eq:MSRD}
\small
\begin{split}
    L_D=&-\mathbb{E}_{(\tau,c_\tau)\sim\mathcal{U}}\left[\log D_{\theta_\text{Task}, \theta_{\text{S-}c_\tau}}(s_t,a_t,s_{t+1})\right]\\
    &-\mathbb{E}_{(\tau,c_\tau)\sim\pi_{\phi_i}}\left[{\log(1-D_{\theta_\text{Task},\theta_{\text{S-}c_\tau}}(s_t,a_t,s_{t+1}))}\right]\\
    &+\alpha \mathbb{E}_{(\tau,c_\tau)\sim\pi_{\phi_i}}\left(\left\vert\left\vert R_{\text{S-}c_\tau}(s_t) \right\vert\right\vert\right)
\end{split}
\end{equation}
\normalsize
Together with the reward training, MSRD optimizes a parameterized policy $\pi_{\phi_{i}}$ for each strategy by maximizing the corresponding strategy combined reward $R_{i}$. 

\section{Method}
In this section, we start by introducing the problem setup and notations. We then provide an overview of the DMSRD framework, and describe in detail two key components of DMSRD: \textit{policy mixture} and \textit{between-class discrimination}. 

\subsection{Problem Setup}
We consider a lifelong learning from heterogeneous demonstration process where demonstrations arrive in sequence, as illustrated in Figure~\ref{fig:DMSRD_illustration}. We denote the $i$-th arrived demonstration as $\tau_i$. Unlike previous methods, we do not assume access to the ground-truth strategy label, $c_{\tau_i}$. Similar to MSRD, DMSRD models a shared task reward $R_\text{Task}$, strategy rewards $R_{\text{S-}j}$, and policies corresponding to each strategy $\pi_{j}$. For simplicity, we denote the parameterized rewards and policies as $R_{\theta_\text{Task}}$, $R_{\theta_{\text{S-}j}}$, and $\pi_{\phi_{j}}$. We note that generally $i\neq j$ as DMSRD does not create a separate strategy for every demonstration. We define the number of strategies created by DMSRD till i-th demonstration as $M_i$ ($M_i \leq i$). $\eta_R(\tau)=\sum_{t=1}^\infty{\gamma^{t-1}R(s_t)}$ is trajectory $\tau$'s discounted cumulative return with reward $R$. 

\subsection{DMSRD}
\begin{algorithm}[t]
\small
\SetKwFunction{MSRD}{MSRD}
\SetKwFunction{AIRL}{AIRL}
\SetKwFunction{PolicyMixtureOptimization}{PolicyMixtureOptimization}
\caption{DMSRD}
\label{algo:DMSRD}
\SetKwInOut{Input}{Input}\SetKwInOut{Output}{Output}
\Input{Task reward $R_{\theta_\text{Task}}$, strategy rewards $R_{\theta_{\text{S-}j}}$, policies $\pi_{\phi_j}$ for $j=\{1,\cdots,M_i\}$, threshold $\epsilon$}
\BlankLine
\While{lifetime learning from heterogeneous demonstration}{
Obtain demonstration $\tau_i$\\
$\vec{w}\leftarrow$\PolicyMixtureOptimization{$\tau_i, \{\pi_{\phi_j}\}$}\\
$D_\text{KL}^\text{mix}\leftarrow \mathbb{E}_{\tau\sim\pi_{\vec{w}}}{D_\text{KL}(\tau_i,\tau)}$\\
\eIf{$D_\text{KL}^\text{mix}<\epsilon$}{
MixtureWeights[i]$\leftarrow\vec{w}$, $M_{i+1}\leftarrow M_i$\\
Update $R_{\theta_\text{Task}}, R_{\theta_{\text{S-}j}}, \pi_{\phi_j}$ by Equation~\ref{eq:D-mix-loss} and \MSRD 
}{
$\pi_\text{new},R_{\theta_{\text{S-}(M_i+1)}}\leftarrow$\AIRL{$\tau_i$}\\
$D_\text{KL}^\text{new}\leftarrow\mathbb{E}_{\tau\sim\pi_\text{new}}{D_\text{KL}(\tau_i,\tau)}$\\
\eIf{$D_\text{KL}^\text{mix}<D_\text{KL}^\text{new}$}{
MixtureWeights[i]$\leftarrow\vec{w}$, $M_{i+1}\leftarrow M_i$\\
Update $R_{\theta_\text{Task}}, R_{\theta_{\text{S-}j}}, \pi_{\phi_j}$ by Equation~\ref{eq:D-mix-loss} and \MSRD
}{
MixtureWeights[i]$\leftarrow[\underbrace{0, 0, 0, \cdots}_{M_i\ \text{zeros}}, 1]$\\
$M_{i+1}\leftarrow M_i+1$, $m_{M_{i+1}}\leftarrow i$\\
Update $R_{\theta_\text{Task}}, R_{\theta_{\text{S-}j}}, \pi_{\phi_j}\ \forall j$ by \MSRD 
}
}
}
\end{algorithm}
With our lifelong learning from demonstration problem setup, when a new demonstration $\tau_i$ becomes available, we seek to address the following key questions: 
\begin{enumerate}
    \item How can we efficiently synthesize a policy that solves the task and personalizes to the demonstration?
    \item How can we incorporate the new demonstration into the algorithm's knowledge so the algorithm could provide accurate, efficient, and scalable adaptation for future demonstrations? 
\end{enumerate}


We present our proposed method, DMSRD, to answer these questions with pseudocode in Algorithm~\ref{algo:DMSRD}. When a new demonstration is available, DMSRD tries to explain it by learned strategies via \textit{policy mixture} optimization (line 3). DMSRD then evaluates the identified mixture according to the trajectory recovery objective (line 4). If the trajectory generated by the mixture policy is close enough to the demonstration, as measured by whether the KL-divergence between the trajectories is within a threshold, $\epsilon$, DMSRD adopts the mixture without considering creating a new strategy to achieve \textit{efficiency} and \textit{scalability}, shown on line 5. 


If the threshold is not met, we create a new strategy by training AIRL with the new demonstration, and compare new strategy's performance on the same KL-divergence objective with respect to the \textit{policy mixture}, shown in lines 9-11. If the mixture performs better, we incorporate the mixture weights as in line 12-13. If not, we incorporate the new strategy into DMSRD's knowledge in lines 15-17. This procedure enables DMSRD to automatically decide whether to utilize already-learned policies to model a new demonstrator vs. synthesizing a new strategy reward and associated policy. 



After determining how to explain the demonstration $\tau_i$ (a mixture-based or new, independent strategy), DMSRD optimizes the MSRD loss (Equation~\ref{eq:MSRD}) for all strategies with their corresponding demonstrations (lines 7, 13, and 17). To learn from demonstrations explained by \textit{policy mixture}, we further propose \textit{Between-Class Discrimination}, a novel approach to train each strategy reward by discriminating its own strategy demonstration from other strategies, including mixtures. To prevent biasing the task reward, we postpone MSRD training until we have two strategies.

\subsection{Policy Mixture Optimization}
In order to achieve efficient personalization for a new demonstration $\tau_i$ and scalability in lifetime, we construct \textit{policy mixture} to model the new demonstration by a linear geometric combination of existing policies $\pi_{1},\pi_{2},\cdots,\pi_{M_i}$, as shown in Equation~\ref{eq:policy_mixture}, where $\mathcal{N}$ denotes gaussian distribution, $\vec{w}$ contains linear combination weights ($w_j\geq 0$,  $\sum_{j=1}^{M_i}w_j=1$), $\mu_{\pi_{\phi_j}}$ is the mean output of policy $j$ (parameterized by neural network), and $\sigma$ controls action variation.
\begin{align}
\label{eq:policy_mixture}
    \pi_{\vec{w}}(\cdot\vert s)=\mathcal{N}(\mu_{\vec{w}}(s), \sigma^2),\quad \mu_{\vec{w}}(s)=\sum_{j=1}^{M_i}{w_j\mu_{\pi_{\phi_j}}(s)}
\end{align}
As the ultimate goal of demonstration modeling is to recover the demonstrated behavior, we optimize the linear weights, $\vec{w}$, by the objective of the divergence between the trajectory induced by the mix policy and the demonstration, illustrated in Equation~\ref{eq:policy_mixture_objective}. 
\begin{align}
\label{eq:policy_mixture_objective}
    \minimize_{w_1,w_2,\cdots,w_{M_i}}{\mathbb{E}_{\tau\sim \pi_{\vec{w}}}\left[\text{Div}(\tau_i,\tau)\right]}
\end{align}
Specifically, we choose Kullback-Leibler divergence (KL-divergence)~\cite{10.1214/aoms/1177729694} on the state distributions of trajectories in our implementation, but other divergence metrics also fit. We estimate the state distribution within a trajectory with a non-parametric Kozachenko-Leonenko estimator~\cite{zbMATH04030688}. 

Due to the non-differentiable nature of the trajectory generation process (the environment transition function is only known by a blackbox simulator), we seek a non-gradient-based optimizer to solve the optimization problem. We choose a na\"ive, random optimization method by generating random weight vectors $\vec{w}$, evaluating Equation~\ref{eq:policy_mixture_objective}, and choosing the weight that achieves the minimization. Empirically, we find this optimization method works effectively as the number of strategies generated by DMSRD is low, making the optimization problem relatively easy. We leave the exploration of more sophisticated non-gradient-based optimization methods such as covariance matrix adaptation evolution strategy~\cite{hansen2006cma} for future work. 

\begin{table*}[th]
\caption{This table shows learned policy metrics of AIRL, MSRD, and DMSRD. The results are averages of three trials. }
\begin{center}
\begin{tabular}{lcccccc}
 \hline
 Domains & \multicolumn{3}{c}{Inverted Pendulum} & \multicolumn{3}{c}{Lunar Lander} \\
 \cline{2-7}
 Methods & AIRL Single & MSRD & DMSRD (Ours) & AIRL Single & MSRD & DMSRD (Ours)\\
 \hline
 Environment Return & $-172.7$ & $-166.4$ & $\mathbf{-38.5}$ & $-11546.0$ & $-9895.3$ & $\mathbf{-2405.5}$\\
 Demonstration Log Likelihood & $-11546.0$ & $-40870.5$ & $\mathbf{-6525.0}$ & $-15137.8$ & $-11124.2$ & $\mathbf{-6723.7}$  \\
 Estimated KL Divergence$^*$ & $4.08$ & $7.67$ & $\mathbf{4.01}$ & $71.4$ & $\mathbf{70.9}$ & $79.9$\\
 \hline
\end{tabular}
\end{center}
\quad $^*$Lower is better
\label{table:result}
\end{table*}

\subsection{Between-Class Discrimination}
\textit{Between-Class Discrimination} (BCD) enforces the strategy reward to discriminate its own strategy demonstration from demonstrations that share certain portion of the strategy (the portion could be zero, meaning completely different strategy). Intuitively, if the mixture $\vec{w}$ for demonstration $\tau_i$ has weight $w_j$ on strategy $j$, we could view the demonstration $\tau_i$ has a portion of $w_j$ with strategy $j$. Thus, under the strategy-only reward $R_{\text{S-}i}$, the probability of $\tau_i$ should be $w_j$ proportion of the probability of the ``pure'' demonstration, $\tau_{m_j}$. Such property could be leveraged to enforce a relationship between the reward given to the ``pure'' 
demonstration, $\tau_{m_j}$, and the reward given to another demonstration, $\tau_i\neq\tau_{m_j}$. We present this idea in Theorem \ref{theorem:trajectory_prob}. 
\begin{theorem}
\label{theorem:trajectory_prob}
Under maximum entropy principal,
\small
\begin{align*}
\frac{P(\tau_i;R_{\theta_{\text{S-}j}})}{{P(\tau_{m_j};R_{\theta_{\text{S-}j}})}}=w_j\Leftrightarrow\exp\left(\eta_{R_{\theta_{\text{S-}j}}}(\tau_i)\right)=w_j\exp\left(\eta_{R_{\theta_{\text{S-}j}}}(\tau_{m_j})\right)
\end{align*}
\normalsize
\end{theorem}
\begin{proof}
The conclusion follows by maximum entropy policy distribution with reward $R_{\theta_{\text{S-}j}}$~\cite{ziebart2008maximum}. 
\end{proof}
Thus, we propose to enforce the relationship of rewards given to pure demonstration and other demonstrations by a MSE loss as shown in Equation~\ref{eq:D-mix-loss}, where $w_{i,j}$ denotes the mixture weight of demonstration $\tau_i$ on strategy $j$. 
{\tiny\begin{align}
\label{eq:D-mix-loss}
L_\text{BCD}(\theta^{\text{S-}1..j};\tau_i)=\sum_{j=1}^{M_i}\left(\exp(\eta_{R_{\theta_{\text{S-}j}}}(\tau_i))-w_{i,j} \exp(\eta_{R_{\theta_{\text{S-}j}}}(\tau_{m_j}))\right)^2
\end{align}
\normalsize}

Intuitively, BCD facilitates robust strategy reward learning by training rewards to recognize demonstrations that share certain portion of the strategy. In addition, \textit{Between-class discrimination} forces strategies to distill any common knowledge amongst them to the task reward.

\section{Results}
\quad We test DMSRD on two simulated continuous control environments in OpenAI Gym~\cite{DBLP:journals/corr/BrockmanCPSSTZ16}: Inverted Pendulum~\cite{todorov2012mujoco} and Lunar Lander~\cite{10.5555/1121584}. We abbreviate Inverted Pendulum as IP and Lunar Lander as LL for result descriptions. The goal in IP is to balance a pendulum by the cart's horizontal movements, with the reward defined as the negative angle between the pendulum and the upright position. The objective in LL is a controlled landing of a space craft. The agent receives a reward of $100$ if it successfully lands, $-100$ if it crashes, $10$ for each leg-ground contact, and $-0.3$ for firing its engine.
To make the environments suitable for our application, we remove termination judgements to allow various behaviors and add a max-time constraint (1000 for IP, 500 for LL). 

\paragraph{Heterogeneous Demonstrations}
As suggested in MSRD, we generate a collection of heterogeneous-strategy demonstrations by jointly optimizing a task and a diversity reward. Diversity is All You Need (DIAYN)~\cite{eysenbach2018diversity} trains a discriminator to distinguish behaviors and the discriminator output entropy is used as our diversity reward. We generate ten unique policies and collect three demonstrations from each policy. Generally, various strategies in IP include sliding or swinging while in LL the spacecraft takes unique flight paths to approach the landing pad\footnote{\label{footnote:video}Videos for strategies available at \url{shorturl.at/pABGK}}. 

\paragraph{Benchmark Experiments}
We compare DMSRD against AIRL and MSRD by running three trials of each method to convergence with the same number of environment samples. On policy metrics, we train a separate AIRL with each demonstration (AIRL Single). On reward metrics, we train a single AIRL on the entire set of demonstrations (AIRL Batch) to improve generalizability on reward learning. 
\subsection{Policy Evaluation}
In this subsection, we show the learned policy comparisons on three metrics: 1) environment returns; 2) demonstration log likelihood evaluated on the learned policy; 3) estimated KL Divergence of the state distributions between the demonstration and rollouts generated by the learned policy\footnote{Similar to how we estimate the state distribution in our method, we apply non-parametric Kozachenko-Leonenko estimator~\cite{zbMATH04030688}.}. The results are summarized in Table \ref{table:result}. All experiment results do not satisfy the assumptions of parametric ANOVA tests, so we perform the non-parametric Friedman test followed by a two-tailed posthoc test. Note that from ten demonstrations, DMSRD created six strategies for each trial in IP and an average of three strategies in LL. 

\paragraph{Learned Policy Task Performance}
DMSRD's learned policies, particularly the \textit{policy mixture} are successful at the task. A Friedman test shows a statistically significant difference amongst the 3 groups with $Q(2)=29.4, p<.01$ for IP and $Q(2)=7.8, p<.05$ for LL. In IP, Nemenyi–Damico–Wolfe (Nemenyi) posthoc test shows a statistically significant difference between AIRL and DMSRD ($p<.01$) and between MSRD and DMSRD ($p<.01$). For LL, a posthoc tests shows a significant difference between AIRL and DMSRD ($p<.05$) and no significant difference between MSRD and DMSRD.

\paragraph{Explanation of Demonstration}
DMSRD is able to explain the demonstrations' strategies more accurately. A Friedman test shows a significant difference when evaluating log likelihood of the demonstration by the learned policy: $Q(2)=16.8, p<.01$ for IP and $Q(2)=22.2, p<.01$ for LL. With IP, the the Nemenyi posthoc test shows the difference between AIRL and DMSRD is not significant and difference between MSRD and DMSRD is significant ($p<.01$). With LL, posthoc test shows a significant difference between AIRL and DMSRD ($p<.01$) and between MSRD and DMSRD ($p<.01$).

\paragraph{Expression of Strategic Preference}
Qualitatively, we find that DMSRD can learn policies and \textit{policy mixtures} that resemble their respective demonstrations\footnotemark[1]. For IP, a Friedman test on the estimated KL divergence shows significant differences between the algorithms $Q(2)=22.2, p<.01$. A posthoc analysis with Nemenyi test shows there is no significant difference between AIRL and DMSRD and a significant difference between MSRD and DMSRD ($p<.01$). In LL, a Friedman test does not show a significant difference $Q(2)=2.4, p=0.30$.

\subsection{Reward Evaluation}
\begin{figure}[t]
\centering
\includegraphics[width=0.9\columnwidth]{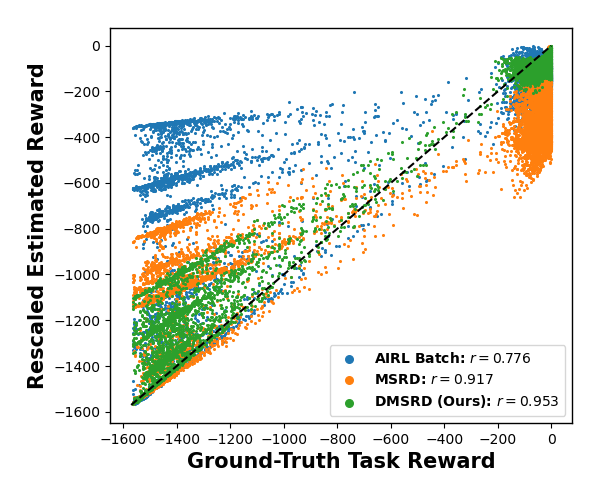}
\caption{This figure shows correlation between the estimated task reward with the ground truth task reward for Inverted Pendulum. Each dot is a trajectory.}
\label{fig:task_reward}
\end{figure}
\begin{figure}[t]
\centering
\includegraphics[width=0.49\columnwidth]{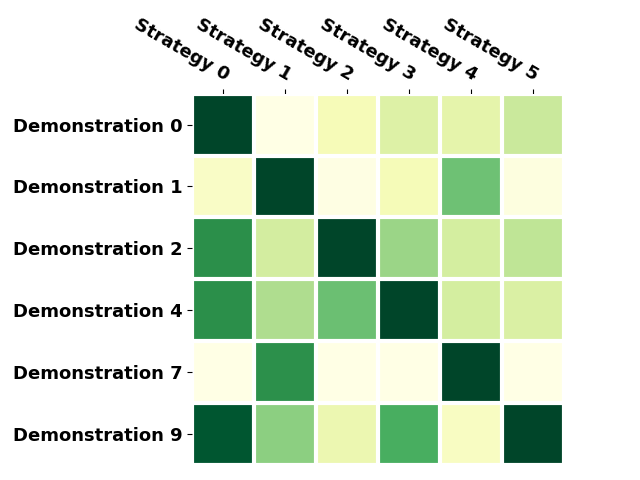}
\includegraphics[width=0.49\columnwidth]{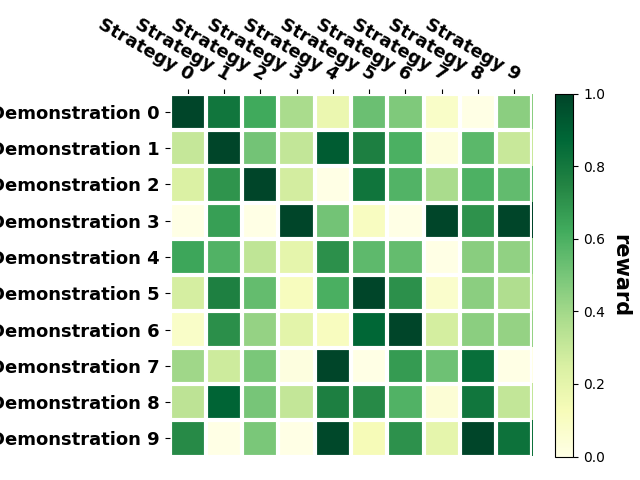}
\caption{This figure depicts evaluation of strategy rewards on demonstrations on Inverted Pendulum for DMSRD (Left) MSRD (Right). Columns are normalized to $[0, 1]$. Darkest squares along the diagonal indicate the strategy rewards identify the respective demonstrations accurately. }
\label{fig:strategy_reward}
\end{figure}
In this subsection, we show the learned task/strategy reward comparisons between DMSRD, MSRD and AIRL. 
\paragraph{Task Reward Correlation}
We construct a test dataset of 10,000 trajectories by the same ``env+diversity reward'' training, obtaining ten unique policies and collecting ten trajectories each every 100 training iterations. Thus, the test dataset has different strategies with varying success. We evaluate the reward function learned by AIRL, MSRD, and DMSRD by comparing the estimated task reward with the ground-truth reward for each trajectory. We compare correlations using a $z$-test after Fischer r-to-z-transformation \cite{fischerrtoz}. For IP, DMSRD has a significantly higher correlation than AIRL ($z=58.56, p<.01$), and than MSRD ($z=20.76, p<.01$). The comparison is also shown Figure~\ref{fig:task_reward}. For LL, DMSRD has a correlation $r=0.614$ which is better than AIRL $r=0.502$ ($z=11.5, p<.01$) and MSRD $r=0.586$ ($z=3.09, p<.01$). DMSRD successfully leverages between-class discrimination to ensure that strategy rewards specialize to the given strategic preferences. As a result, it distills common knowledge to the task reward, allowing it to yield a more accurate reward function.


\paragraph{Strategy Reward Heatmap}
In order to test the success of the strategy rewards in discriminating specific strategic preferences, we evaluate each demonstration on the strategy rewards. For IP, we could observe in Figure~\ref{fig:strategy_reward} (left) each DMSRD strategy reward gives the highest reward to the corresponding demonstration, demonstrating success of DMSRD in identifying the latent preference of the ``pure'' demonstrations. Over three trials, the percentage of strategies DMSRD recognizes is $100\%$ compared to MSRD, which recognizes $57\%$ demonstrations. Likewise, in LL, DMSRD recognizes $100\%$ pure demonstrations compared to MSRD ($83\%$).

\subsection{Lifelong Learning Experiment}
In this subsection, we show DMSRD's success in a larger-scale lifelong learning scenario. We design 95 mixtures with randomized weights from 5 base policies for a total of 100 demonstrations and provide these to DMSRD with the 5 base policies being first. We perform two trials of DMSRD and find it learns a concise set of 5 strategies in the first and 9 strategies in the second that explain the scope of behaviors displayed across the 100 demonstrations. Table \ref{table:scale} shows a similarly successful policy and reward performance as previous experiments. DMSRD's ability to utilize the knowledge of learned strategies facilitates scalable learning of new demonstrations. The strong task reward correlation shows DMSRD can mitigate catastrophic forgetting~\cite{kirkpatrick2017overcoming}, retaining the knowledge it learned while also mitigating capacity saturation since it dynamically expands if presented with new strategies. Thus, we show that DMSRD can scale well to facilitate successful lifelong learning. 


\begin{table}[t]
\begin{center}
\caption{This table shows policy and reward evaluation metrics of DMSRD in the lifelong learning experiment. The results are averaged over two trials. }
\begin{tabular}{ |c|c|  }
 \hline
 Metrics & DMSRD\\
 \hline
 Cumulative Return &  -8.03\\
 Estimated KL Divergence & 12.5\\
 Log Likelihood & -2298.4\\
 Task Reward Correlation & 0.889\\
 Strategies Identified & 100\%\\
 \hline
\end{tabular}
\label{table:scale}
\end{center}
\end{table}

\section{Limitations, Future Work, \& Conclusion}
\quad Limitations of DMSRD include 1) if the initial demonstrations are not diverse enough, the task reward learning may be biased and the ability to effectively model a large number of demonstrations would be impacted; 2) DMSRD inherits the property from AIRL that the learned reward functions are non-stationary. 

We plan to test DMSRD in a real robot task with demonstrations from users with different latent preferences. Being able to learn personalized intentions could inspire future work to optimally combine different strategies to generate novel trajectories. Other further work includes studying how to recover a minimally spanning strategy set that can explain all demonstrations. 

In this paper, we present DMSRD, a novel lifelong learning from heterogeneous demonstration approach. In benchmarks against AIRL and MSRD, we demonstrate DMSRD's \textit{flexibility} to adapt to personal preferences and \textit{efficiency} by distilling common knowledge to the task reward and constructing \textit{policy mixture} when possible. \textit{Between-class discrimination} further helps DMSRD learn a more accurate task and strategy rewards. We also illustrate DMSRD's \textit{scalability} in how it learns a concise set of strategies to solve the problem of explaining a large number of demonstrations. 


\bibliography{aaai22}


\end{document}